\documentclass[11pt]{article}

\usepackage[margin=1in]{geometry}

\usepackage{amsmath,amssymb,amsthm}

\usepackage{graphicx}

\usepackage[authoryear,round]{natbib}

\usepackage{authblk}

\usepackage[
    colorlinks=true,
    citecolor=blue,
    linkcolor=blue,
    urlcolor=blue
]{hyperref}

\newtheorem{theorem}{Theorem}[section]
\newtheorem{lemma}[theorem]{Lemma}
\newtheorem{proposition}[theorem]{Proposition}
\newtheorem{corollary}[theorem]{Corollary}

\theoremstyle{definition}
\newtheorem{definition}[theorem]{Definition}

\theoremstyle{remark}
\newtheorem{remark}[theorem]{Remark}

\let\originalproof\proof
\let\endoriginalproof\endproof

\RenewDocumentEnvironment{proof}{o}
{
    \IfNoValueTF{#1}
        {\originalproof}
        {\originalproof[Proof of #1]}
}
{
    \endoriginalproof
}

\providecommand{\acks}[1]{%
    \section*{Acknowledgments}
    #1
}

\usepackage{makecell}
\usepackage{multirow}

\usepackage{mathtools}
\usepackage{interval}
\usepackage{physics}
\usepackage{braket}

\usepackage{fourier}

\newcommand{\du}{\mathrm{d}}
\newcommand{\eu}{\mathrm{e}}
\newcommand{\KL}{D_{\text{KL}}}

\newcommand{\ER}{\operatorname{ER}}
\newcommand{\regret}{\operatorname{Regret}}

\newcommand{\transp}{\mathsf{\scriptscriptstyle{T}}}

\title{Regret Lower Bound for Online Learning of Quantum States with Logarithmic Loss, with Applications to Quantum State Tomography and Confidence Sequences for Matrix Means}

\author[1,2,3]{Yen-Huan~Li}
\author[1]{Wei-Fu~Tseng}

\affil[1]{Department of Computer Science and Information Engineering, National Taiwan University}
\affil[2]{Department of Mathematics, National Taiwan University}
\affil[3]{Center for Quantum Science and Engineering, National Taiwan University}

\date{}

\begin{document}

\maketitle

\begin{abstract}
We study the problem of online learning of quantum states with logarithmic loss, which is a quantum generalization of online portfolio selection and has applications to quantum state tomography, the task of estimating an unknown quantum state from measurement outcomes.
We first establish a regret lower bound, showing that the regret rates of Exponential Weights and the algorithm of \citet{Tseng2025} have optimal polynomial dependence on the dimension.
Our analysis leads to two further results.
First, our lower-bound proof leads to a quantum state tomography algorithm that is sample optimal up to polylogarithmic factors in the sample size. 
Second, we develop an algorithm for constructing confidence sequences for the mean of a matrix-valued random variable.
A numerical experiment shows that the latter algorithm is competitive with existing methods.
\end{abstract}

\bibliographystyle{plainnat}

\section{Introduction} \label{sec_intro}

In the mathematical formalism of quantum mechanics, a quantum state is represented by a \emph{density matrix}, which is a Hermitian positive semidefinite matrix with unit trace. 
A measurement is represented by a \emph{positive operator-valued measure (POVM)}, which is a collection of Hermitian positive semidefinite matrices that sum to the identity matrix.
Given a density matrix $\rho$ and a POVM $\set{ M_\iota \mid \iota \in \mathcal{I} }$ indexed by a set $\mathcal{I}$, the measurement outcome is a random variable $\eta$ taking values in $\mathcal{I}$ whose distribution is given by the \emph{Born rule}:
\begin{equation}
\mathsf{P} \left( \eta = \iota \right) = \tr ( M_\iota \rho ), \quad \text{for all } \iota \in \mathcal{I}, \label{eq_born}
\end{equation}
For convenience, we will write $\mathcal{D}_d$ for the set of $d \times d$ density matrices. 

Consider the following multi-round game of \emph{online learning of quantum states with logarithmic loss} between two players, Physicist and Nature.
The game consists of $T$ rounds.
At each round $1 \leq t \leq T$:
\begin{itemize}
\item Physicist announces a density matrix $\rho_t \in \mathcal{D}_d$;
\item Nature then announces a non-zero $d \times d$ Hermitian positive semidefinite matrix $A_t$ whose eigenvalues lie in\footnote{One can get rid of the upper limit $1$ for the eigenvalues. Nevertheless, since the loss function is logarithmic, scaling $A_t$ does not affect the regret value, so we can always scale $A_t$ such that its eigenvalues are all at most $1$.} $\interval{0}{1}$;
\item finally, Physicist suffers the loss $f_t(\rho_t) \coloneqq - \log \tr(A_t \rho_t)$. 
\end{itemize}
Physicist aims to achieve small regret against any strategy of Nature. 
The regret is defined as
\[
\regret_T \coloneqq \sum_{t=1}^T f_t(\rho_t) - \min_{\rho \in \mathcal{D}_d} \sum_{t=1}^T f_t(\rho) . 
\]
For convenience, throughout this paper we refer to this game simply as \emph{online learning of quantum states (OLQS)}, although online learning of quantum states has also been studied under other loss functions \citep{Aaronson2018,Aaronson2019,Bansal2025,Chen2024a,Yang2020}.

The game was independently proposed and studied by \citet{Lin2021b} and \citet{Zimmert2022}.
The work of \citet{Lin2021b} was motivated by quantum state tomography, the task of estimating an unknown quantum state from measurement outcomes \citep{Paris2004a}.
Suppose that a fixed POVM is used for the measurements and that $M_t$ is the POVM element corresponding to the outcome observed at round $t$.
Then, $f_t(\rho)$ is precisely the negative log-likelihood of the observed outcome under the candidate state $\rho$.
Thus, OLQS can be viewed as an online formulation of maximum-likelihood quantum state tomography \citep{Hradil1997}.
In addition, it is easily checked that any no-regret algorithm for OLQS, together with an online-to-batch conversion \citep{Hazan2023,Orabona2026}, can be used as a stochastic optimization method for approximating the maximum-likelihood estimate in quantum state tomography.

OLQS is also closely connected to online portfolio selection introduced by \citet{Cover1991}, as noted by both \citet{Lin2021b} and \citet{Zimmert2022}. 
Indeed, suppose that the matrices $M_t$ are all diagonal in a fixed common basis.
Since $\tr(M_t\rho)$ then depends only on the diagonal entries of $\rho$ in this basis, Physicist may, without loss of generality, restrict attention to density matrices that are diagonal in the same basis.
Identifying the diagonal entries of $\rho_t$ with a portfolio vector $x_t$ in the probability simplex and those of $M_t$ with a return vector $a_t \in \interval{0}{1}^d$, we obtain $f_t( \rho_t ) = - \log \braket{ a_t, x_t }$, which is exactly the logarithmic loss in online portfolio selection.

Since neither the logarithmic loss nor its gradient is Lipschitz continuous, regret guarantees of textbook online convex optimization algorithms do not directly apply to online portfolio selection. 
Indeed, optimal online portfolio selection, in both regret rate and computational complexity, has been an open problem for more than three decades. 
\citet{Cover1996} and \citet{Ordentlich1998} established that the minimax-optimal regret rate is $O ( d \log T )$ and this rate can be achieved by the Universal Portfolios algorithm, an instance of Exponential Weights \citep{Hazan2007,Hoeven2018,Vovk1998a}. 
Unfortunately, the computational complexity of Universal Portfolios scales poorly with the time horizon \citep{Kalai2002}, with a per-round runtime of $O ( d^4 T^{14} )$ \citep{Luo2018}. 

Table~\ref{table_algs} summarizes existing state-of-the-art algorithms for online portfolio selection and their counterparts for OLQS.
The best-known algorithm for online portfolio selection is perhaps VB-FTRL \citep{Jezequel2026}, which achieves a regret rate of $O(d \log(d+T))$ with a per-round runtime of $O(d^2 T)$.
When $T \gg d$, the regret rate of VB-FTRL is already optimal, but the dependence of its per-round runtime on $T$ becomes undesirable.

\begin{table}[t]
    \centering
    \begin{tabular}{|c|c|c|c|}
    \hline
    \multicolumn{2}{|c|}{Online Portfolio Selection}
    & \multicolumn{2}{c|}{Online Learning of Quantum States} \\
    \hline

    Algorithm & Regret Rate & Algorithm & Regret Rate \\
    \hline

    \makecell{Universal Portfolios \\ \citep{Cover1991}}
    & $O(d \log T)$
    & \makecell{Exponential Weights \\ \citep{Zimmert2022}}
    & $O(d^2 \log T)$ \\
    \hline

    \makecell{VB-FTRL \\ \citep{Jezequel2026}}
    & $O(d \log(d + T))$
    & \makecell{Q-VB-FTRL \\ \citep{Tseng2025}}
    & $O(d^2 \log(d + T))$ \\
    \hline
    
    \makecell{PAE+DONS \\ \citep{Mhammedi2022}}
    & $O(d^2 \log^5 T)$
    & \multicolumn{2}{c|}{No known quantum counterpart.} \\
    \hline

    \makecell{BISONS \\ \citep{Zimmert2022}}
    & $O(d^2 \log^2 T)$
    & \makecell{\textsc{Schr\"odinger}'s BISONS \\ \citep{Zimmert2022}}
    & $O(d^3 \log^2 T)$ \\
    \hline

    \makecell{Barrier subgradient method \\ \citep{Nesterov2011}}
    & \multirow[c]{3}{*}[-3ex]{$\widetilde{O}(\sqrt{dT})$}
    & \makecell{Q-LB-OMD \\ \citep{Tsai2023}}
    & \multirow[c]{3}{*}[-3ex]{$\widetilde{O}(\sqrt{dT})$} \\
    \cline{1-1}\cline{3-3}

    \makecell{Soft-Bayes \\ \citep{Orseau2017}}
    &
    & \makecell{Q-Soft-Bayes \\ \citep{Lin2021b}}
    & \\
    \cline{1-1}\cline{3-3}

    \makecell{LB-FTRL \\ \citep{Tsai2023a}}
    &
    & \makecell{Q-LB-FTRL \\ \citep{Tsai2023b}}
    & \\
    \hline
	\end{tabular}
    \caption{Existing algorithms for online portfolio selection and online learning of quantum states with logarithmic loss (OLQS), together with their regret upper bounds. Rows group related online portfolio selection algorithms and their quantum counterparts. Here, $d$ denotes the number of assets in online portfolio selection and the dimension of density matrices in OLQS, and $T$ denotes the time horizon. The notation $\widetilde O$ suppresses polylogarithmic factors in \(d\) and \(T\). LB-FTRL and Q-LB-FTRL in the last row use linearized losses. }
    \label{table_algs}
\end{table}

A trade-off between regret and computational complexity is achieved by PAE+DONS \citep{Mhammedi2022} and BISONS \citep{Zimmert2022}, which attain
regret rates of $O(d^2 \operatorname{polylog}(T))$ with per-round runtimes of $O(\operatorname{poly}(d)\operatorname{polylog}(T))$.
Another trade-off between regret and computational complexity is provided by the barrier subgradient method \citep{Nesterov2011}, Soft-Bayes \citep{Orseau2017}, and LB-FTRL \citep{Tsai2023a}.
These algorithms achieve regret rates of $\widetilde{O}(\sqrt{dT})$ with per-round runtimes of $\widetilde{O}(d)$.

Following the convention in quantum information theory, we will refer to online portfolio selection and OLQS as the \emph{classical case} and the
\emph{quantum case}, respectively.
As shown in Table~\ref{table_algs}, existing algorithms for OLQS are generalizations of their classical counterparts.
Such generalizations can be technically challenging because of the noncommutative nature of the quantum setting.

We observe an interesting separation in Table~\ref{table_algs}:
\begin{itemize}
\setlength{\emergencystretch}{2em}
\item Q-LB-OMD \citep{Tsai2023}, Q-Soft-Bayes \citep{Lin2021b}, and Q-LB-FTRL \citep{Tsai2023b} inherit the same regret rates as their classical counterparts.
\item Exponential Weights \citep{Zimmert2022}, Q-VB-FTRL \citep{Tseng2025}, and \textsc{Schrodinger}'s~BISONS \citep{Zimmert2022} suffer an additional factor of $d$ in their regret bounds compared with their classical counterparts.
\end{itemize}

Both phenomena can be explained.
Q-LB-OMD, Q-Soft-Bayes, and Q-LB-FTRL were developed so that their algorithms and regret bounds reduce to those of their classical counterparts when all
matrices involved are diagonal.
In contrast, the analyses of Exponential Weights, Q-VB-FTRL, and \textsc{Schr\"{o}dinger}’s BISONS use a vectorization argument that treats a $d \times d$ density matrix as a vector of dimension $\Theta( d^2 )$, leading to the additional factor of $d$ in their regret upper bounds.


The aforementioned observation naturally leads to the following question: 
\begin{quote}
Is the additional factor of $d$ in the regret bound for Exponential Weights in OLQS intrinsic to the problem, or is it merely an artifact of the analysis?  
\end{quote}
The answer to this question also has implications for the significance of the Q-VB-FTRL algorithm.
No explicit implementation of Exponential Weights for OLQS is currently available in the literature.
Q-VB-FTRL, on the other hand, admits an implementation with
$O(T\operatorname{poly}(d))$ per-round runtime while achieving an $O(d^2 \log(d+T))$ regret bound.
Therefore, if the additional factor of $d$ in these regret bounds could be removed, then the significance of Q-VB-FTRL as a computationally efficient algorithm with near-optimal dependence on the dimension would also be called into question.

In this paper, we establish the following regret lower bound for OLQS.

\begin{theorem} \label{thm_main}
Suppose that $T \gg d^2 \log d$.
Then no algorithm for OLQS can achieve a worst-case regret bound of $O(d^\gamma \log T)$ for any $\gamma < 2$.
\end{theorem}

This result establishes that the regret rates of both Exponential Weights and Q-VB-FTRL have the correct polynomial dependence on the dimension.

Our proof strategy exploits the connection between OLQS and quantum state tomography. 
We first develop an online learning approach to quantum state tomography, reducing quantum state tomography to OLQS via an online-to-batch conversion, and derive a high-probability upper bound on the estimation error in terms of the regret of the OLQS algorithm used. 
Then, existing high-probability lower bounds on the estimation error from the quantum information theory literature \citep{Acharya2025,Chen2023,Haah2017,Lowe2025} immediately imply that the regret of any OLQS algorithm must be sufficiently large. 

Our result also yields two applications. 
\begin{enumerate}
\item As a byproduct of our proof strategy, we develop a novel algorithm for quantum state tomography that is sample-optimal up to polylogarithmic factors in the sample size.

\item Generalizing the connection between confidence sequences for estimating the mean of a random vector and online portfolio selection, we develop an algorithm that constructs confidence sequences for the mean of a random matrix.
Theorem~\ref{thm_main} establishes that the regret bound for Exponential Weights has optimal polynomial dependence on the dimension, motivating our use of this bound in the confidence-sequence construction in Section~\ref{sec_confidence_sequence}.
\end{enumerate}

\paragraph{Notation}
The $d \times d$ identity matrix is denoted by $I_d$. 
The transpose and conjugate transpose of a matrix $M$ are denoted by $M^\transp$ and $M^\ast$, respectively. 
The trace norm of a matrix $M$ is denoted by $\norm{ M }_1$, as it coincides with the Schatten $1$-norm. 
The $1$-norm of a vector $v$ is denoted by $\norm{ v }_1$. 
For any $d$-dimensional probability vectors $p$ and $q$, their Kullback-Leibler divergence is denoted by $\KL ( p, q )$ and defined as 
\[
\KL ( p, q ) \coloneqq \sum_{i = 1}^d p[i] \log \frac{p[i]}{q[i]} , 
\]
where $p[i]$ and $q[i]$ denote the $i$-th entries of $p$ and $q$, respectively.

\section{Preliminaries on Quantum State Tomography and Quantum Measurements}


Quantum state tomography refers to the task of estimating an unknown density matrix from measurements and measurement outcomes, and is essential for building reliable quantum computing devices.
Throughout this paper, we will only consider the following \emph{non-adaptive, single-copy} formulation of quantum state tomography, a commonly studied setting because of its relative ease of implementation in practice: 
Suppose we have $n$ independent copies of an unknown density matrix $\rho^\natural \in \mathcal{D}_d$. 
We measure each of the copies independently using a fixed POVM\footnote{Some quantum state tomography methods adopt random Pauli measurements \citep{Flammia2012,Gross2010} or random Pauli basis measurements \citep{Riofrio2017}, where a measurement is selected at random from a given collection. It is easily checked that the resulting randomized measurement is mathematically equivalent to a single POVM.} $\mathcal{M} = \set{ M_\iota | \iota \in \mathcal{I} } \subset \mathbb{C}^{d \times d}$, for some finite index set $\mathcal{I}$, and observe the measurement outcomes $\eta_1, \ldots, \eta_n$ as independent and identically distributed (i.i.d.) random variables taking values in $\mathcal{I}$. 
Then, according to the Born rule \eqref{eq_born}, the negative log-likelihood function is given by 
\[
L ( \rho; \iota ) \coloneqq - \log \tr ( M_\iota \rho ) , \quad \text{for all } \iota \in \mathcal{I} , 
\]
and the maximum-likelihood estimate is given by minimizing the empirical average of the negative log-likelihood functions over the data (measurement outcomes). 
In this paper, however, we will consider an online learning approach to quantum state tomography, rather than the maximum-likelihood one. 

We will measure the performance of any estimate $\hat{\rho}_n$, based on $n$ copies of the unknown density matrix, using the trace-norm error $\norm{ \hat{\rho}_n - \rho^\natural }_1$. 
Under this performance measure, a sample complexity lower bound is already known. 

\begin{theorem}[\citet{Haah2017,Lowe2025}] \label{thm_optimal_sample_complexity}
For $0 < \varepsilon \leq 1 / 2$, any quantum state tomography algorithm using nonadaptive, single-copy measurements that estimates every $d \times d$ density matrix to trace-norm error at most $\varepsilon$ with probability at least $2/3$ requires $n = \Omega ( d^3 / \varepsilon^2 )$ copies of the unknown state. 
\end{theorem}

The sample complexity is highly dependent on the POVM used. 
In this paper, we will consider the POVM induced by \emph{complex projective $4$-designs}, whose \emph{POVM norm constant} will be useful in our analysis. 

\begin{definition}[Complex projective $t$-design]
Let $S_d \coloneqq \set{ x\in\mathbb{C}^d | x^\ast x=1}$, and let $\mathrm{d} \mathcal{S}_d$ denote its normalized unitarily invariant probability measure.
For any positive integer $t$, a complex projective $t$-design of dimension $d$ and cardinality $m$ \citep{Ambainis2007,Matthews2009,Scott2006} consists of a finite set $\set{ x_1, \ldots, x_m } \subset \mathcal{S}_d$ such that
\[
\sum_{i = 1}^m \frac{1}{m} ( x_i x_i^\ast )^{\otimes \, t} = \int_{x \in \mathcal{S}_d} ( x x^\ast )^{\otimes \, t} \, \du \mathcal{S}_d , 
\]
where $\otimes$ denotes the tensor product. 
The induced POVM is given by the set of operators 
\[
\set{ ( d / m ) x_1 x_1^\ast, \ldots, ( d / m ) x_m x_m^\ast } . 
\]
\end{definition}

\begin{definition}[POVM norm constant \citep{Kueng2016,Matthews2009}] \label{def_povm_norm_constant}
For any POVM $\mathcal{M} = \set{ M_\iota | \iota \in \mathcal{I} }$, where $\mathcal{I}$ is a finite index set, the associated \emph{POVM norm constant} is defined as
\[
\lambda_{\mathcal{M}} = \inf_{\rho \neq \sigma \atop \rho, \sigma \in \mathcal{D}_d } \frac{\sum_{i \in \mathcal{I} } \abs{ \tr \left( M_i \rho \right) - \tr \left( M_i \sigma \right) } }{ \norm{ \rho - \sigma }_1 } . 
\]
\end{definition}

\begin{theorem}[{\citet[(23)]{Matthews2009}}] \label{lem_povm_norm_constant}
The POVM norm constant of the POVM induced by complex projective $4$-designs is at least $1 / ( 3 \sqrt{d} )$. 
\end{theorem}

\begin{remark}
\citet{Matthews2009} proved that the largest POVM norm constant is $\Theta ( 1 / \sqrt{d} )$. 
The POVM induced by complex projective $4$-designs is hence optimal in this sense. 
Indeed, \citet{Matthews2009} showed that the \emph{uniform POVM} also attains this optimal rate. 
Nevertheless, defining the uniform POVM formally requires the general measure-theoretic definition of POVM, as it has an infinite outcome space. 
For the purpose of this paper, it suffices to consider the POVM induced by complex projective $4$-designs, which has a finite outcome space. 
\end{remark}

Define the risk function as the expected negative log-likelihood: 
\[
R ( \rho ) \coloneqq \mathsf{E}_{\eta_1} \left[ L ( \rho; \eta_1 ) \right] .
\]
The excess risk of any estimate $\hat{\rho}_n$---not necessarily the maximum-likelihood one---is then given by
\[
\ER ( \hat{\rho}_n ) \coloneqq R ( \hat{\rho}_n ) - \min_{\rho \in \mathcal{D}_d} R ( \rho ) . 
\]
We will use the following inequality in analysing the sample complexity of the proposed online learning-based quantum state tomography algorithm. 

\begin{proposition} \label{prop_basic_inequality}
In any non-adaptive, single-copy quantum state tomography experiment with a fixed POVM $\mathcal{M}$ satisfying $\lambda_{\mathcal M} > 0$, any estimate $\hat{\rho}_n$ must satisfy
\[
\norm{ \hat{\rho}_n - \rho^\natural }_1 \leq \frac{\sqrt{ 2 \ER ( \hat{\rho}_n ) }}{\lambda_{\mathcal{M}}} . 
\]
\end{proposition}

In particular, if the POVM is the one induced by complex projective $4$-designs, then 
\[
\norm{ \hat{\rho}_n - \rho^\natural }_1 \leq 3 \sqrt{ 2 d \ER ( \hat{\rho}_n ) } . 
\]
Proposition \ref{prop_basic_inequality} should be familiar to experts in related fields. 
Nevertheless, we do not find an appropriate reference, so we provide a proof for completeness. 

\begin{proof}[Proposition \ref{prop_basic_inequality}]
A direct calculation gives
\[
\ER ( \hat{\rho}_n ) = \KL ( p_{\rho^\natural} , p_{\hat{\rho}_n} ) , 
\]
where $p_\rho$ denotes the probability distribution of the measurement outcome when the measured density matrix is $\rho$. 
That is, $p_\rho \coloneqq ( \tr ( M_\iota \rho ) )_{\iota \in \mathcal{I}}$, which is a probability vector in $\mathbb{R}^{\abs{ \mathcal{I} }}$. 
By Pinsker's inequality and the definition of POVM norm constants (Definition \ref{def_povm_norm_constant}), we write 
\[
\lambda_{\mathcal{M}} \norm{ \hat{\rho}_n - \rho^\natural }_1 \leq \norm{ p_{\hat{\rho}_n} - p_{\rho^\natural} }_1 \leq \sqrt{ 2 \ER ( \hat{\rho}_n ) } ,
\]
which proves the proposition. 
\end{proof}

\section{Regret Lower Bound for Online Learning of Quantum States and a Quantum State Tomography Algorithm} \label{sec_online_qst}

This section proposes an online learning approach to quantum state tomography. 
According to Proposition \ref{prop_basic_inequality}, the estimation error is controlled by the excess risk; 
hence, it appears sufficient to implement an OLQS algorithm followed by an online-to-batch conversion, which translates the algorithm's regret guarantee into a high-probability excess risk guarantee. 
Nevertheless, it is easily checked that the logarithmic loss $f_t ( \rho ) \coloneqq - \log \tr ( A_t \rho )$ in OLQS can be arbitrarily large near the boundary of $\mathcal{D}_d$, so standard online-to-batch conversion based on the Azuma-Hoeffding inequality does not directly apply. 
Moreover, the logarithmic loss is $1$-exp-concave, so we expect to achieve a fast rate of $\tilde{O} ( 1 / T )$ for the excess risk \citep{Koren2015,Mehta2017}. 
Recall the definition of exp-concavity: 

\begin{definition}[Exp-concavity \citep{Hazan2007,Kivinen1999}]
A function $f$ is called $\eta$-exp-concave if the function $\eu^{- \eta f}$ is concave. 
\end{definition}

We show below that the online-to-batch conversion scheme for exp-concave losses proposed by \citet{Mehta2017} suffices. 
Although this online-to-batch conversion scheme requires the loss functions to be bounded, this issue can be solved if we shrink the density matrices slightly toward the normalized identity matrix $I_d / d$, called the maximally mixed state in quantum information theory. 
While \citet{Foster2018} also use this shrinkage in a computationally more expensive online-to-batch conversion based on the work of \citet{Mehta2017}, we show that it is unnecessary to use their scheme. 

For any $\rho \in \mathcal{D}_d$, define the shrunk density matrix
\[
\tilde{\rho} \coloneqq \left( 1 - \frac{1}{n} \right) \rho + \frac{I_d}{n d} . 
\] 
Also define the set of shrunk density matrices: 
\[
\tilde{\mathcal{D}}_{d} \coloneqq \set{ \tilde{\rho} | \rho \in \mathcal{D}_d } . 
\]

Consider the following online learning-based scheme for quantum state tomography. 
\begin{itemize}
\item Let there be $n$ rounds in the OLQS game. 
\item Let the matrices $A_t$ in the OLQS be set as $M_{\eta_t} / \tr ( M_{\eta_t} )$, where $\eta_t$ are measurement outcomes as defined in the previous section. 
      Note that trace normalization does not affect the regret value since the loss functions are logarithmic. 
\item Implement an online learning algorithm in this OLQS game, and obtain iterates $\rho_1, \rho_2, \ldots, \rho_n$. 
\item Output the empirical average of shrunk iterates $\hat{\rho}_n \coloneqq ( \tilde{\rho}_1 + \tilde{\rho}_2 + \cdots + \tilde{\rho}_n ) / n$. 
\end{itemize}

\begin{lemma} \label{lem_surrogate_error}
Suppose that $n \geq 2$. 
For any $\rho \in \mathcal{D}_d$ such that $\tr ( A_t \rho ) > 0$ for all $t$, it holds that 
\[
f_t ( \tilde{\rho} ) - f_t ( \rho ) \leq \frac{1}{n - 1} , \quad \text{for all } 1 \leq t \leq n , 
\]
and
\[
\min_{\tilde{\rho} \in \tilde{\mathcal{D}}_d} R ( \tilde{\rho} ) - \min_{\rho \in \mathcal{D}_d} R ( \rho ) \leq \frac{1}{n - 1} . 
\]
\end{lemma}

\begin{proof}
Let $\mu = 1 / n$. 
We write 
\[
f_t ( \tilde{\rho} ) - f_t ( \rho ) = \log \left( \frac{ \tr ( A_t \rho ) }{ ( 1 - \mu ) \tr ( A_t \rho ) + \frac{\mu}{d} \tr ( A_t ) } \right) \leq \log \left( \frac{1}{1 - \mu} \right) .  
\]
By the inequality above, we write 
\[
\min_{\tilde{\rho} \in \tilde{\mathcal{D}}_d} R ( \tilde{\rho} ) \leq R ( \rho ) + \log \left( \frac{1}{1 - \mu} \right) . 
\]
The lemma then follows from the fact that 
\[
\log \left( \frac{1}{1 - \mu} \right) = \log \left( \frac{n}{n - 1} \right) = \log \left( 1 + \frac{1}{n - 1} \right) \leq \frac{1}{n - 1} . 
\]
\end{proof}

Lemma \ref{lem_surrogate_error} implies that
\[
\widetilde{\regret}_n \coloneqq \sum_{t = 1}^n f_t ( \tilde{\rho}_t ) - \min_{\tilde{\rho} \in \tilde{\mathcal{D}}_d } \sum_{t = 1}^n f_t ( \tilde{\rho} ) \leq \regret_n ( \tilde{\rho} ) + \frac{n }{n - 1} . 
\]
Define $\widetilde{\regret}_n^{+} \coloneq \max\set{ \widetilde{\regret}_n, 0 }$. 
Note that the loss functions $f_t$ are bounded by $\log ( n d )$ and $1$-exp-concave over $\tilde{D}_d$. 
Applying the online-to-batch conversion scheme of \citet[Corollary 2]{Mehta2017} and Lemma \ref{lem_surrogate_error}, we obtain that for any $n \geq 3$, 
\begin{align*} 
& R \left( \frac{\tilde{\rho}_1 + \cdots + \tilde{\rho}_n}{n} \right) - \min_{\rho \in \mathcal{D}_d} R ( \rho ) \\
& \quad = R \left( \frac{\tilde{\rho}_1 + \cdots + \tilde{\rho}_n}{n} \right) - \min_{\tilde{\rho} \in \tilde{\mathcal{D}}_d} R ( \tilde{\rho} ) + \min_{\tilde{\rho} \in \tilde{\mathcal{D}}_d} R ( \tilde{\rho} ) - \min_{\rho \in \mathcal{D}_d} R ( \rho ) \\
& \quad \leq \frac{\widetilde{\regret}_n^+}{n} + 4 \frac{\sqrt{ \widetilde{\regret}_n^+ }}{n} \sqrt{ \left( 1 + \log \left( n d \right) \right) \log \left( \frac{4 \log n}{\delta} \right) } + \\
& \quad \quad \, \, 16 \left( 1 + \log \left( n d \right) \right) \frac{\log \left( \frac{4 \log n}{\delta} \right)}{n} + \frac{1}{n - 1} . 
\end{align*}
with probability at least $1 - \delta$. 

Setting $\delta = 1 / 3$ and plugging in the regret rate bounds for Exponential Weights and Q-VB-FTRL, we immediately obtain the following high-probability estimation error bound, which is nearly optimal according to Theorem \ref{thm_optimal_sample_complexity}. 

\begin{theorem} \label{thm_estimation_error_bound}
Suppose that $n \gg d$. 
If $n \geq 3$, the proposed online learning-based scheme for quantum state tomography with either Exponential Weights or Q-VB-FTRL, initialized at $\rho_1 = I_d / d$, yields an estimate $\hat{\rho}_n$ that satisfies 
\begin{align*}
\norm{ \hat{\rho}_n - \rho^\natural }_1 = \tilde{O} \left( \sqrt{ \frac{d^3 \log n}{n} } \right)
\end{align*}
with probability at least $2 / 3$. 
Here, the notation $\tilde{O}$ suppresses polylogarithmic factors in $\log (n)$. 

In addition, the resulting algorithm is sample-optimal up to polylogarithmic factors in $n$. 
\end{theorem}

Now, we are ready to prove Theorem \ref{thm_main}. 

\begin{proof}[Theorem \ref{thm_main}]
Suppose that there exists an OLQS algorithm with a regret rate of $O ( d^\gamma \log T )$ for some $\gamma < 2$. 
Then, adopting this algorithm with the arguments above, there exists a quantum state tomography algorithm that achieves an estimation error of $\tilde{O} ( \sqrt{ \left( d^{1 + \gamma} \log n \right) / n } )$ with probability at least $2/3$, which violates the known sample complexity lower bound in Theorem \ref{thm_optimal_sample_complexity}. 
\end{proof}


\begin{remark} \label{rem_optimal_regret}
Note that Theorem \ref{thm_main} does not determine the minimax regret rate of OLQS.
In particular, it rules out regret bounds of the form $O(d^\gamma \log T)$ for $\gamma<2$, but does not establish a matching $\Omega(d^2 \log T)$ lower bound for the regret bound achieved by Exponential Weights.
Since OLQS generalizes online portfolio selection, whose minimax regret rate is $\Theta(d \log T)$, the minimax regret of OLQS is also $\Omega(d \log T)$.
However, this lower bound, even when combined with Theorem \ref{thm_main}, is insufficient to establish that the minimax regret rate of OLQS is $\Theta(d^2 \log T)$.
We leave the determination of the exact minimax regret rate of OLQS to future work.
\end{remark}

Consider a real version of the OLQS game, in which the matrices $A_t$ are real and the action set $\mathcal{D}_d$ is replaced by $\mathcal{D}_d^{\mathbb{R}}$, the set of real symmetric positive semidefinite matrices with unit trace. 

\begin{corollary} \label{cor_main_real}
Theorem \ref{thm_main} continues to hold for the real version of the OLQS game. 
\end{corollary}

The corollary will be used in the next section to justify our approach to constructing confidence sequences for estimating the mean of a real random matrix.
Its proof, deferred to Appendix \ref{app_main_real}, shows that any algorithm for the real version in dimension $2d$ induces an algorithm for the original OLQS game in dimension $d$ whose worst-case regret is no greater than that of the real algorithm.

\section{Confidence Sequence for Random Matrix Mean Estimation} \label{sec_confidence_sequence}

Let $(\xi_t)_{t \in \mathbb{N}}$ be a stochastic process satisfying
$\mathsf{E}[\xi_t \mid \xi_1, \ldots, \xi_{t-1}] = \mu^\natural$
for every $t \in \mathbb{N}$, where $\mu^\natural$ is an unknown parameter.
We say that a sequence of sets $(\mathcal{C}_t)_{t \in \mathbb{N}}$ is a \emph{$(1-\alpha)$-confidence sequence} for $\mu^\natural$ if each $\mathcal{C}_t$ depends only on $\xi_1, \ldots, \xi_t$ and
\[
\mathsf{P}\left(\mu^\natural \in \mathcal{C}_t
\text{ for all } t \in \mathbb{N}\right) \geq 1-\alpha.
\]
Thus, confidence sequences may be viewed as an anytime-valid extension of confidence sets.
The reader is referred to, for example, the works of \citet{Howard2021}, \citet{Ramdas2022}, and \citet{Ramdas2025a} for historical reviews and expositions of related concepts.

We are particularly interested in a relatively recent line of research that constructs confidence sequences by associating the task with online portfolio selection \citep{Orabona2024,Ryu2024,Ryu2024a}. 
\citet{Orabona2024} and \citet{Ryu2024a} considered only the case where $\xi_t$ are scalar-valued random variables, as their constructions explicitly rely on special properties of Universal Portfolios in the 2-dimensional case.
\citet{Ryu2024} considered the case where $\xi_t$ are vector-valued random variables by reducing the construction of confidence sequences to solving general online portfolio selection problems. 

We aim to extend this line of research to the case where $\xi_t$ are matrix-valued random variables.
Although confidence sequences for the mean of a random matrix can be constructed by vectorizing the observations and applying the algorithm of \citet{Ryu2024},
the numerical result in Section \ref{sec_numerical} shows that the resulting confidence sequences yield larger estimated confidence-set volumes than our construction.

Suppose that $\xi_t$ are $\mathcal{D}_d$-valued random variables. 
Define the following ``matrix division'' between $\xi_t$ and $\mu$: 
\[
\xi_t ( \mu ) \coloneqq \mu^{-\dagger/2} \xi_t \mu^{-\dagger/2} , \quad \text{for all } \mu \in \mathcal{D}_d , 
\]
where $\mu^{-\dagger / 2}$ denotes the square root of the Moore-Penrose generalized inverse of the matrix $\mu$. 

For every $\mu \in \mathcal{D}_d$, consider the OLQS game in which Nature announces $A_t = \xi_t ( \mu )$ in round $t$, and denote the output of Physicist in round $t$ by $\rho_t ( \mu )$. 
Consider the stochastic process $( W_t ( \mu ) )_{t \geq 0}$ such that $W_0 = 1$ and 
\[
W_t ( \mu ) = W_{t - 1} ( \mu ) \tr \left( \xi_t ( \mu ) \rho_t ( \mu ) \right) , \quad \text{for all } t \in \mathbb{N}, 
\]
which is non-negative because both $\xi_t ( \mu )$ and $\rho_t ( \mu )$ are positive semi-definite. 
Since $\rho_{t + 1} ( \mu^\natural )$ is completely determined by $\xi_1, \ldots, \xi_t$, we have 
\begin{align*}
\mathsf{E} \left[ W_{t + 1} ( \mu^\natural ) \, \middle\vert \, \xi_1, \ldots, \xi_t \right]
& = W_t ( \mu^\natural ) \tr \left( \left( \mu^\natural \right)^{-\dagger / 2} \mu^\natural \left( \mu^\natural \right)^{-\dagger / 2} \rho_{t + 1} ( \mu^\natural ) \right) \\
& \leq W_t ( \mu^\natural ) . 
\end{align*}
Hence, $( W_t ( \mu^\natural ) )_{t \geq 0}$ is a non-negative supermartingale. 

\begin{theorem}[Ville's inequality \citep{Ville1939}]
Let $( \eta_t )_{t \geq 0}$ be a non-negative supermartingale. Then, it holds that 
\[
\mathsf{P} \left( \sup_{t \geq 0} \eta_t \geq c \right) \leq \frac{\mathsf{E} \left[ \eta_0 \right]}{c} , \quad \text{for any } c > 0 . 
\]
\end{theorem}

Applying Ville's inequality to the stochastic process $( W_t ( \mu^\natural ) )_{t \geq 0}$, we immediately get the following theorem. 

\begin{theorem}
The sequence $( \mathcal{C}_t )_{t \in \mathbb{N}}$ defined as
\begin{equation}
\mathcal{C}_t = \Set{ \mu \in \mathcal{D}_d | W_t ( \mu ) < \frac{1}{\alpha} } ,  \label{eq_confidence_sequence}
\end{equation}
is a $( 1 - \alpha )$-confidence sequence for $\mu^\natural$. 
\end{theorem}

\begin{remark}
Suppose that the observations $\xi_t$ are almost surely diagonal. 
Then $\mu^\natural$ is also diagonal, so we may restrict the candidate means to diagonal density matrices.
In this case, our construction specializes to that of \citet{Ryu2024}. 
\end{remark}

Define the regret for the sequence of observations $( \xi_t ( \mu ) )_{t \in \mathbb{N}}$ as
\[
\regret_t ( \mu ) \coloneqq \sum_{\tau = 1}^t \left( - \log \tr ( \xi_\tau ( \mu ) \rho_\tau ( \mu ) ) \right) - L_t^\star ( \mu ) , \quad L_t^\star ( \mu ) \coloneqq \min_{\rho \in \mathcal{D}_d} \sum_{\tau = 1}^t \left( - \log \tr ( \xi_\tau ( \mu ) \rho ) \right) .
\]
Then, the confidence set $\mathcal{C}_t$ \eqref{eq_confidence_sequence} can be expressed equivalently as
\[
\mathcal{C}_t = \Set{ \mu \in \mathcal{D}_d | \regret_t ( \mu ) > \log \alpha - L_t^\star ( \mu ) } .
\]
The quantity $L_t^\star ( \mu )$ can be efficiently approximated using existing algorithms \citep{Li2019a,Lin2021,Tsai2024}.
Computing $\regret_t ( \mu )$, however, requires implementing an OLQS algorithm.
Given Theorem \ref{thm_main}, Exponential Weights and Q-VB-FTRL are natural choices, but both are computationally expensive.


\begin{corollary} \label{cor_larger_confidence_sequence}
The sequence of sets $( \overline{\mathcal{C}}_t )_{t \in \mathbb{N}}$ given by 
\begin{equation}
\overline{\mathcal{C}}_t \coloneqq \Set{ \mu \in \mathcal{D}_d | \overline{\regret_t ( \mu )} > \log \alpha - L_t^\star ( \mu ) } , \label{eq_confidence_sequence_enlarged}
\end{equation}
where $\overline{\regret_t ( \mu )}$ is any upper bound on $\regret_t ( \mu )$ realized by any OLQS algorithm, is a $(1 - \alpha)$-confidence sequence for $\mu^\natural$. 
\end{corollary}

The construction does not require implementing any OLQS algorithm but requires a tight regret bound. 
Exploiting exp-concavity of logarithmic losses, the result of \citet[Theorem 7]{Hazan2007} implies that the regret of Exponential Weights is no larger than $( d^2 - 1 ) ( 1 + \log ( T + 1 ) )$ in the OLQS game. 
Below, we provide a sharper regret bound whose proof, deferred to Appendix \ref{app_blum_kalai}, essentially follows the strategy of \citet[Theorem 1]{Blum1999}.

\begin{theorem} \label{thm_blum_kalai}
Let $\mathcal{X}$ be a bounded closed convex set with non-empty relative interior in a finite-dimensional Euclidean space. 
Let $p$ be the dimension of its affine hull. 
Consider an online convex optimization problem of $T$ rounds, where the action set is $\mathcal{X}$ and the loss functions are $\eta$-exp-concave for some $\eta > 0$. 
Then, Exponential Weights with learning rate $\eta$ and the uniform prior achieves a regret no greater than $(1 / \eta) \log \binom{T + p}{p}$. 
\end{theorem}

Combining Corollary \ref{cor_larger_confidence_sequence} and Theorem \ref{thm_blum_kalai}, we obtain the following explicit constructions of confidence sequences for both the complex and real cases. 
The proof is deferred to Appendix \ref{app_explicit_confidence_sequences}. 

\begin{corollary} \label{cor_explicit_confidence_sequences}
The sequence of sets $( \overline{\mathcal{C}_t} )_{t \in \mathbb{N}}$ defined in Corollary \ref{cor_larger_confidence_sequence} with $\overline{\regret_t ( \mu )} = \log \binom{t + p}{p}$, where $p = d^2 - 1$, is a $(1 - \alpha)$-confidence sequence. 

If it is known that the random matrices $\xi_t$ are almost surely real, then we can set $p = d^2 / 2 + d / 2 - 1$ and replace $\mathcal{D}_d$ with $\mathcal{D}_d^{\mathbb{R}}$ in \eqref{eq_confidence_sequence_enlarged} and the definition of $L_t^\star ( \mu )$. 
\end{corollary}

\begin{remark}
Our construction is not limited to the case where $\xi_t$ are random density matrices. 
More generally, suppose that $\xi_t$ are Hermitian matrices satisfying $A \leq \xi_t \leq B$ for known Hermitian matrices $A$ and $B$. 
If $A \neq B$, define $c \coloneqq \tr ( B - A )$ and the affine transformation
\[
\xi_t \mapsto \widetilde{\xi}_t
\coloneqq
\begin{pmatrix}
\frac{\xi_t - A}{c} & 0 \\
0 & 1 - \tr \left( \frac{\xi_t - A}{c} \right)
\end{pmatrix} \in \mathcal{D}_{d + 1}. 
\]
A confidence sequence for $\mathsf{E} \left[ \widetilde{\xi}_t \right]$ then immediately yields one for $\mathsf{E} \left[ \xi_t \right]$ through the inverse affine transformation.
\end{remark}

\section{Numerical Experiment} \label{sec_numerical}

We conduct a numerical experiment to compare our confidence sequence construction with existing methods for covariance matrix estimation.
Let $(v_t)_{t\in\mathbb{N}}$ be i.i.d. mean-zero random vectors in $\mathbb{R}^d$ with covariance matrix $\Sigma$.
Our goal is to construct a confidence sequence for $\Sigma$.
This problem fits our framework by setting $\xi_t=v_t v_t^\transp$ and $\mu^\natural = \Sigma$.
Figure \ref{fig_anisotropic} compares our construction with those of \citet[Eqs. (15)--(16)]{Ryu2024}, implemented using matrix vectorization and the regret bound of \citet{Cover1996}, \citet[Eq. (4.10)]{Howard2021}, and \citet[Theorem 4.2]{Wang2025}.
The latter two methods are based on matrix concentration inequalities rather than online-learning regret bounds.

We set $d=8$ and consider $T=200$ observations, with confidence level $1-\alpha=0.95$.
The random vectors $( v_t )_{t \in \mathbb{N}}$ are generated such that $\norm{ v_t }_2 = 1$, the covariance matrix $\Sigma$ lies in $\mathcal{D}_d^{\mathbb{R}}$, and the support of the distribution of $\xi_t = v_t v_t^\transp$ contains non-commuting matrices.

The performance is measured by the \emph{relative Hilbert--Schmidt volume}, defined as the ratio between the Hilbert--Schmidt volumes of a confidence set and $\mathcal{D}_d^{\mathbb{R}}$ \citep{Zyczkowski2003}.
For each method, we intersect its confidence sets obtained up to time $t$ and restrict the resulting set to $\mathcal D_d^{\mathbb R}$, to ensure that the resulting sets are nested over time.
We approximate these volumes by Monte Carlo sampling and average the estimates over 1000 independent repetitions of the experiment.
The error bars in Figure \ref{fig_anisotropic} represent pointwise 95\% Student-$t$ confidence intervals for the mean relative volume.

\begin{figure}[t]
    \centering
    \includegraphics[width=\textwidth]{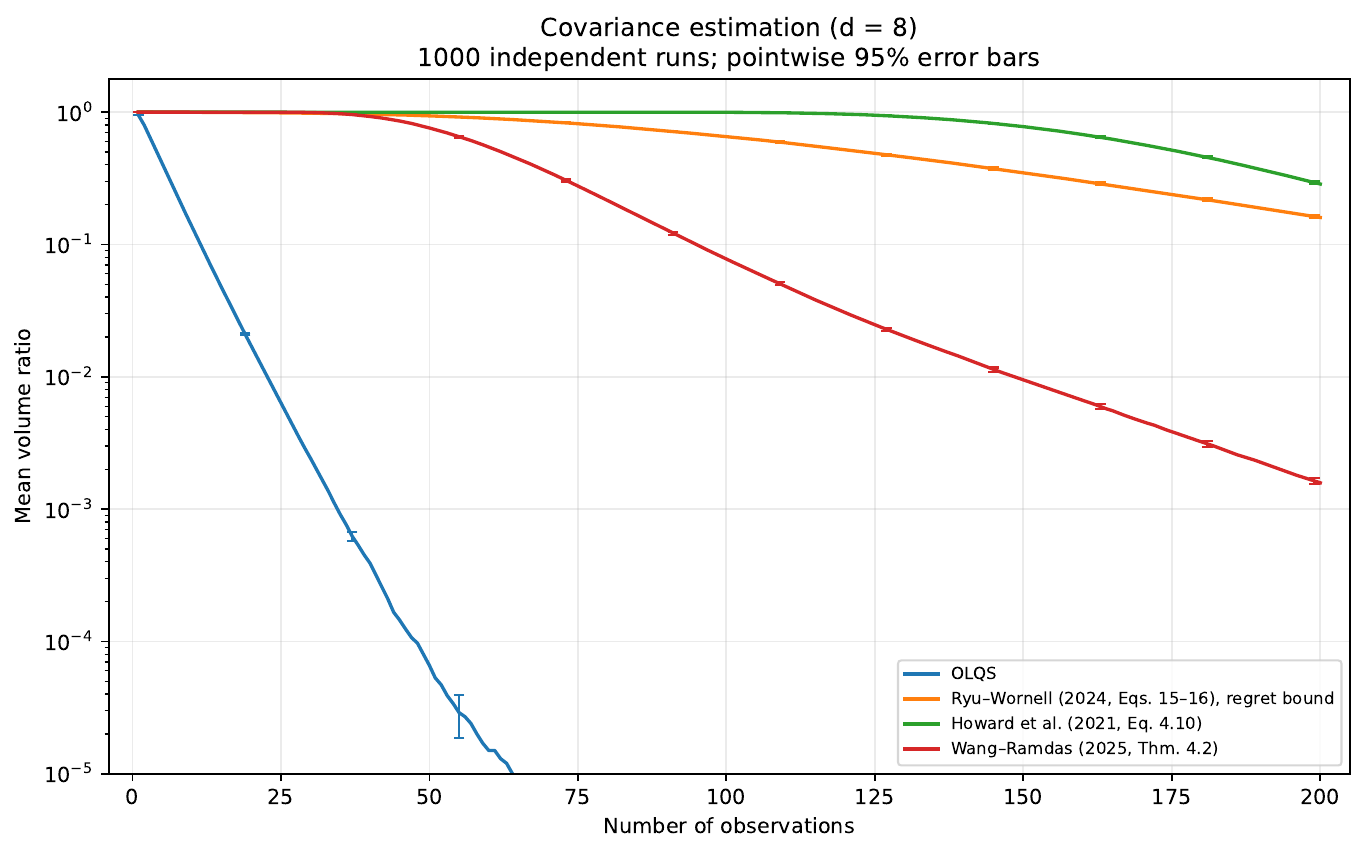}
    \caption{Comparison of the mean relative Hilbert--Schmidt volumes of confidence sequences for covariance matrix estimation.
    OLQS denotes our construction following the second claim of Corollary \ref{cor_explicit_confidence_sequences}.
    See Section \ref{sec_numerical} and Appendix \ref{app_numerical} for details.}
    \label{fig_anisotropic}
\end{figure}

Figure \ref{fig_anisotropic} shows that our construction yields substantially smaller estimated relative volumes than the three competing implementations in this experiment.
In particular, the construction based on matrix vectorization and the regret bound of \citet{Cover1996} produces considerably larger estimated confidence-set volumes, suggesting that this adaptation of \citet{Ryu2024} is less effective in the considered setting.
For \citet{Howard2021} and \citet{Wang2025}, we follow the parameter choices specified in their papers without further tuning.
In comparison, our construction requires no tuning parameter beyond the confidence level. 

Further details of the numerical experiment can be found in Appendix \ref{app_numerical}. 

\section{Concluding Remarks}

We have shown that no algorithm for OLQS can achieve a regret bound of $O ( d^\gamma \log T )$ for any $\gamma < 2$, establishing the near-optimality of Exponential Weights and the Q-VB-FTRL algorithm of \citet{Tseng2025} in terms of the polynomial dependence of their regret rates on $d$.
Our proof strategy also yields a novel online learning-based quantum state tomography algorithm that is sample-optimal up to polylogarithmic factors in
the sample size. 
Furthermore, viewing OLQS as a quantum generalization of online portfolio selection, we have extended the connection between online portfolio selection and confidence sequence construction to matrix-valued parameters.
The regret bound for Exponential Weights yields confidence sequences whose candidate membership can be assessed by convex optimization without implementing an OLQS algorithm.
Figure \ref{fig_anisotropic} illustrates the empirical competitiveness of our construction in covariance matrix estimation.

Several research problems follow.
First, our result characterizes only the optimal polynomial dependence of regret rate on the dimension and does not determine the minimax-optimal regret rate, as discussed in Remark \ref{rem_optimal_regret}.
Deriving the minimax-optimal regret rate is a natural direction for future work.
In particular, deriving sharper regret upper bounds, ideally with optimal constants, could lead to tighter confidence sequences via Corollary \ref{cor_larger_confidence_sequence}.
Second, although our confidence sequence appears competitive empirically, we have not established any convergence rate in terms of the sizes of its confidence sets, nor have we proved its optimality under any criterion.
Addressing these questions would help clarify the theoretical advantages of online learning-based confidence sequence construction.
Finally, the online learning-based approach may extend to more general parameter spaces by considering losses of the form $-\log L ( x )$, where $L$ is a linear functional positive on the parameter space.
Developing a unified theory for such constructions would be an interesting direction for future work.

\section*{AI Disclosure}
The strategies for deriving the regret lower bounds and constructing the confidence sequences were developed by the authors.
Nevertheless, the original online-to-batch scheme adopted by the authors was that of \citet{Foster2018}; GPT-6 Pro noted that the scheme of \citet{Mehta2017} suffices.
The regret bound originally derived by the authors for Theorem \ref{thm_blum_kalai} was $( d / \eta ) \log ( T + 1 ) + ( 1 / \eta )$, looser than the current bound. 
GPT-6 Pro noted that the analysis of \citet{Blum1999} can be extended to exp-concave online convex optimization, yielding an improved regret bound.
GPT-6 Pro generated the Julia code for the numerical experiments, while the experimental settings and parameter choices were determined by the authors.
Appendix \ref{app_numerical} was initially generated by GPT-6 after multiple rounds of conversations and then modified by the authors. 
The authors also used GPT-6 to correct grammatical errors in the draft.

\acks{This work was supported by the National Science and Technology Council of Taiwan under Grant Nos. NSTC 112-2628-E-002-019-MY3 and NSTC 115-2221-E-002-208, by the research project “Geometry of Quantum Learning and Optimization” at National Taiwan University under Grant No. NTU-CC-115L893706, and by the Young Theoretical Scientist Award of the Mathematics Division of the National Center for Theoretical Sciences of Taiwan.}

\bibliography{list_olqs_qst_cs}

\appendix

\section{Proof of Corollary \ref{cor_main_real}} \label{app_main_real}

We prove the corollary by showing that any algorithm for the real version of the OLQS game in dimension $2d$ can be transformed into an algorithm for the original OLQS game in dimension $d$, and that both algorithms incur the same regret values on the corresponding data sequences. 
Then, any regret lower bound of the original OLQS game also implies a regret lower bound of its real version.

Let $d \in \mathbb{N}$. 
Denote by $\mathcal{D}_{2d}^{\mathbb{R}}$ the set of $(2d) \times (2d)$ real symmetric positive semidefinite matrices with unit trace. 
For any matrix $M \in \mathbb{C}^{d \times d}$, define the realification map as
\[
\Phi ( M ) \coloneqq 
\begin{pmatrix}
\Re M & - \Im M \\
\Im M & \Re M 
\end{pmatrix} , 
\]
where $\Re M$ and $\Im M$ denote the real and imaginary parts of the matrix $M$, respectively. 

\begin{lemma}[{\citet[Section 3]{Goemans2004}}] \label{lem_goemans}
A Hermitian complex matrix $M$ is positive semidefinite if and only if $\Phi ( M )$ is positive semidefinite. 
\end{lemma}

For any matrix
\[
\rho^{\mathbb{R}} = 
\begin{pmatrix}
P & Q \\
Q^\transp & R 
\end{pmatrix} 
\in \mathcal{D}_{2d}^{\mathbb{R}} , 
\]
where $P$ and $R$ are symmetric matrices in $\mathbb{R}^{d \times d}$, define its ``complexification'' as 
\[
\Psi ( \rho^{\mathbb{R}} ) \coloneqq (P + R) - i \left( Q - Q^\transp \right) , 
\]
where $i$ denotes $\sqrt{-1}$. 

\begin{lemma} \label{lem_remain_density}
The matrix $\Psi ( \rho^{\mathbb{R}} )$ is a density matrix. 
\end{lemma}

\begin{proof}
It is easily checked that the matrix $\Psi ( \rho^{\mathbb{R}} )$ is Hermitian with unit trace. 
Note that for every $v \in \mathbb{C}^d$, 
\[
v^\ast \Psi ( \rho^{\mathbb{R}} ) v = v^\ast ( P + R ) v - i v^\ast ( Q - Q^\transp ) v = \tilde{v}^\ast \rho^{\mathbb{R}} \tilde{v} , 
\]
where 
\[
\tilde{v} \coloneqq 
\begin{pmatrix}
v \\
- i v 
\end{pmatrix} . 
\]
Since \(\rho^{\mathbb R}\) is positive semidefinite, $\tilde{v}^\ast \rho^{\mathbb{R}} \tilde{v} \geq 0$.
Hence, the matrix $\Psi ( \rho^{\mathbb{R}} )$ is also positive semidefinite. 
\end{proof}

\begin{lemma} \label{lem_equal_trace}
For any Hermitian matrix $M \in \mathbb{C}^{d \times d}$ and $\rho^{\mathbb{R}} \in \mathcal{D}^{\mathbb{R}}_{2d}$, it holds that 
\[
\tr ( M \Psi ( \rho^{\mathbb{R}} ) ) = \tr ( \Phi ( M ) \rho^{\mathbb{R}}  ) . 
\]
\end{lemma}

\begin{proof}
Use the notations above. 
We write 
\begin{align*}
\tr ( M \Psi ( \rho^{\mathbb{R}} ) ) & = \tr \left( ( \Re M ) ( P + R ) + ( \Im M ) ( Q - Q^\transp ) \right) + \\
& \quad \, \, i \tr \left( ( \Im M ) ( P + R ) - ( \Re M ) ( Q - Q^\transp ) \right) \\
& = \tr \left( ( \Re M ) ( P + R ) +  ( \Im M ) ( Q - Q^\transp ) \right) \\
& = \tr ( \Phi ( M ) \rho^{\mathbb{R}}  ) . 
\end{align*}
The second equality follows from the fact that both $M$ and $\Psi ( \rho^{\mathbb{R}} )$ are Hermitian, so the imaginary part of the trace must vanish; 
the third equality follows from direct calculation. 
\end{proof}

\begin{lemma} \label{lem_onto}
The mapping $\Psi$ is onto. 
\end{lemma}

\begin{proof}
For any $\rho \in \mathcal{D}_d$, define 
\[
\rho^{\mathbb{R}} \coloneqq \frac{1}{2}
\begin{pmatrix}
\Re \rho & - \Im \rho       \\
\Im \rho & \Re \rho
\end{pmatrix} .  
\]
Note that $\rho^{\mathbb{R}}$ is symmetric, positive semidefinite by Lemma \ref{lem_goemans}, and unit-trace, so it lies in $\mathcal{D}_{2d}^{\mathbb{R}}$. 
In addition, it is easily checked that $\Psi ( \rho^{\mathbb{R}} ) = \rho$. 
\end{proof}

For any algorithm $\mathcal{A}^{\mathbb{R}}$ for Physicist in the real version of the OLQS game in dimension $2d$, consider the following algorithm for Physicist in the original OLQS game in dimension $d$: 
At each round \(t\), obtain the next prediction \(\rho_t^{\mathbb R}\) from \(\mathcal A^{\mathbb R}\) and announce \(\rho_t=\Psi(\rho_t^{\mathbb R})\). 
After Nature reveals \(A_t\), feed \(\Phi(A_t)\) to \(\mathcal A^{\mathbb R}\). 
Note that this is a valid input: 
The matrix $\Phi ( A_t )$ is positive semidefinite by Lemma \ref{lem_goemans}; 
applying Lemma \ref{lem_goemans} again, we write 
\[
I_{2d} - \Phi ( A_t ) = \Phi ( I - A_t ) \geq 0 , 
\]
which shows that all eigenvalues of $\Phi ( A_t )$ are at most $1$. 

Lemma \ref{lem_remain_density} ensures that $\rho_t$ is a density matrix, so the induced algorithm for the original OLQS game is valid.
In addition, Lemma \ref{lem_equal_trace} ensures that the cumulative losses in the original OLQS game and its real version are the same; 
Lemma \ref{lem_onto} ensures that the minimum cumulative losses in hindsight in the original OLQS game and its real version over $\mathcal{D}_d$ and $\mathcal{D}_{2d}^{\mathbb{R}}$, respectively, are also the same. 
Hence, the two algorithms incur the same regret on corresponding data sequences $A_1, \ldots, A_T$ and $\Phi ( A_1 ), \ldots, \Phi ( A_T )$. 
A regret rate of \(O( d^\gamma\log T)\) for any $\gamma < 2$ for the real version of OLQS would therefore imply a regret rate of $O(d^\gamma\log T)$ for the original OLQS game, contradicting Theorem \ref{thm_main}. 

\section{Proof of Theorem \ref{thm_blum_kalai}} \label{app_blum_kalai}

Let us specify the algorithm and define the necessary notations first. 
Denote by $f_t$ the loss function at round $t$. 
For any Lebesgue-measurable $\mathcal{S} \subseteq \mathcal{X}$, let $\operatorname{vol} ( \mathcal{S} )$ denote its Lebesgue measure on the affine hull of $\mathcal{X}$, identified as $\mathbb{R}^p$. 
Exponential Weights with the uniform prior iterates as follows. 
\begin{itemize}
\item Let $\pi_1$ be the density of the uniform distribution over $\mathcal{X}$; 
      that is, $\pi_1 ( x ) = 1 / \operatorname{vol} ( \mathcal{X} )$ for all $x \in \mathcal{X}$. 
\item For every $t \in \mathbb{N}$, output $x_t = \mathsf{E}_{X \sim \pi_t} \left[ X \right]$ and update 
      $$
      \pi_{t + 1} ( x ) \propto \pi_t ( x ) \mathrm{e}^{- \eta f_t ( x )} , \quad \text{for all } x \in \mathcal{X} . 
      $$
\end{itemize}

Then, we prove the theorem. 
Define $f_{1: T} ( x ) \coloneqq f_1 ( x ) + \cdots + f_T ( x )$. 
The iteration rule of Exponential Weights gives the telescoping identity
$$
\prod_{t = 1}^T \int_{\mathcal{X}} \mathrm{e}^{- \eta f_t ( x )} \pi_t ( x ) \, \mathrm{d} x = \int_{\mathcal{X}} \prod_{t = 1}^T \mathrm{e}^{- \eta f_t ( x )} \pi_1 ( x ) \, \mathrm{d} x = \int_{\mathcal{X}} \mathrm{e}^{- \eta f_{1: T} ( x )} \pi_1 ( x ) \, \mathrm{d} x. 
$$
By the $\eta$-exp-concavity of $f_t$, we write 
$$
\mathrm{e}^{- \eta \sum_{t = 1}^T f_t ( x_t )} = \prod_{t = 1}^T \mathrm{e}^{- \eta f_t ( x_t )} \geq \prod_{t = 1}^T \int_{\mathcal{X}} \mathrm{e}^{- \eta f_t ( x )} \pi_t ( x ) \, \mathrm{d} x . 
$$
Let $U$ be a random variable with the uniform probability distribution over $\mathcal{X}$. 
Then, we write 
$$
\int_{\mathcal{X}} \mathrm{e}^{- \eta f_{1: T} ( x )} \pi_1 ( x ) \, \mathrm{d} x = \mathsf{E} \left[ \mathrm{e}^{- \eta f_{1: T} ( U )} \right] . 
$$
We next derive a lower bound on the expectation.
Let $x_\star$ be any comparator in $\mathcal{X}$. 
For any $0 \leq \alpha \leq 1$, define
$$
\mathcal{X}_\alpha \coloneqq \Set{ \alpha x_\star + ( 1 - \alpha ) x | x \in \mathcal{X } } . 
$$
Note that $\operatorname{vol} ( \mathcal{X}_\alpha ) = ( 1 - \alpha )^p \operatorname{vol} ( \mathcal{X} )$. 
In addition, for any $x_\alpha \in \mathcal{X}_\alpha$, by the $\eta$-exp-concavity of $f_t$, we have 
$$
\mathrm{e}^{-\eta f_{1: T} ( x_\alpha )} \geq \alpha^T \mathrm{e}^{- \eta f_{1: T} ( x_\star )}. 
$$
Hence, we obtain
$$
\mathsf{P} \left( \mathrm{e}^{- \eta f_{1: T} ( U )} \geq \alpha^T \mathrm{e}^{- \eta f_{1: T} ( x_\star )} \right) \geq ( 1 - \alpha )^p . 
$$
Since $0 \leq \mathrm{e}^{- \eta f_{1: T} ( U )}$ almost surely, we write 
\begin{align*}
\mathsf{E} \left[ \mathrm{e}^{- \eta f_{1: T} ( U )} \right] &\geq \int_0^{\mathrm{e}^{- \eta f_{1: T} ( x_\star )}} \mathsf{P} \left( \mathrm{e}^{- \eta f_{1: T} ( U )} \geq t' \right) \, \mathrm{d} t' \\
&= \mathrm{e}^{- \eta f_{1: T} ( x_\star )} \int_0^1 \mathsf{P} \left( \frac{\mathrm{e}^{- \eta f_{1: T} ( U )}}{\mathrm{e}^{- \eta f_{1: T} ( x_\star )}} \geq t \right) \, \mathrm{d} t \\
&\geq \mathrm{e}^{- \eta f_{1: T} ( x_\star )} \int_0^1 ( 1 - t^{1/T} )^p \, \mathrm{d} t . 
\end{align*}
The integral in the last line is an instance of the Euler integral and equals $1 / \binom{T + p}{p}$. 
Combining the results above, we obtain
$$
\mathrm{e}^{- \eta \sum_{t = 1}^T f_t ( x_t )} \geq \frac{ \mathrm{e}^{- \eta f_{1: T} ( x_\star )} }{ \binom{T + p}{p} } . 
$$
Taking logarithms yields
$$
\sum_{t=1}^T f_t(x_t) \leq f_{1:T}(x_\star) + \frac{1}{\eta} \log \binom{T+p}{p}.
$$
Since $x_\star\in\mathcal X$ is arbitrary, taking the infimum over $x_\star$ proves the theorem.

\section{Proof of Corollary \ref{cor_explicit_confidence_sequences}} \label{app_explicit_confidence_sequences}

The corollary is obtained by applying the regret bound for Exponential Weights in Theorem \ref{thm_blum_kalai} to Corollary \ref{cor_larger_confidence_sequence}.
We regard the matrix spaces as vector spaces over $\mathbb{R}$ and apply Theorem \ref{thm_blum_kalai} in the affine hull of each decision set.

It is easily checked that $\mathcal{D}_d$ is a bounded closed convex set with non-empty relative interior and that the logarithmic loss is $1$-exp-concave. 
To specify a density matrix, it suffices to specify its diagonal entries, which needs $( d - 1 )$ real numbers given the trace constraint, as well as its strictly upper-triangular entries, which needs $( d^2 - d )$ real numbers. 
Hence, the dimension of the affine hull of $\mathcal{D}_d$ is at most $d^2 - 1$. 
The first claim of the corollary then follows from Theorem \ref{thm_blum_kalai}. 

For the real case, it is easily checked that $\mathcal{D}_d^{\mathbb{R}}$ is a bounded closed convex set with non-empty relative interior, the dimension of whose affine hull is at most $d^2 / 2 + d / 2 - 1$. 
Since the conditional mean $\mu^\natural$ must be real when the observations $\xi_t$ are real almost surely, we have 
\[
\mathsf{P} \left( \mu^\natural \in \overline{\mathcal{C}}_t \text{ for all } t \right) = \mathsf{P} \left( \mu^\natural \in \mathbb{D}_d^{\mathbb{R}} \cap \overline{\mathcal{C}}_t \text{ for all } t \right) , 
\]
showing that we can replace $\mathcal{D}_d$ with $\mathbb{D}_d^{\mathbb{R}}$. 
This restricts the parameter space to real density matrices and allows us to use the smaller affine dimension in the regret bound.
Finally, note that if the minimum cumulative loss $L_T^\star$ is attained at a (complex) density matrix $\rho_T^\star$, then it is also attained at the point
\[
\tilde{\rho}_T^\star \coloneqq \frac{ \overline{\rho_T^\star} + \rho_T^\star }{2} , 
\]
where $\overline{\rho_T^\star}$ denotes the conjugate of $\rho_T^\star$, because for any real symmetric positive semidefinite matrix $M$, we have
\[
\tr ( M \rho_T^\star ) = \tr ( M \tilde{\rho}_T^\star ) . 
\]
Hence, we can use the regret bound of Exponential Weights over $\mathcal{D}_d^{\mathbb{R}}$ and replace $\mathcal{D}_d$ with $\mathcal{D}_d^{\mathbb{R}}$ in the definition of $L_T^\star ( \mu )$. 

\section{Details of Numerical Experiment} \label{app_numerical}

For \citet[Eqs. (15)--(16)]{Ryu2024}, we map each real density matrix to a vector with 35 coordinates in $[0,1]$, which is then converted into a probability vector with 36 coordinates using their construction.
Since evaluating Universal Portfolios is computationally expensive in high dimensions, we instead apply Corollary \ref{cor_larger_confidence_sequence} using the regret bound of \citet[Theorem 2]{Cover1996} for Universal Portfolios with a Dirichlet-$(1/2,\ldots,1/2)$ prior.
For both our method and this adaptation of \citet{Ryu2024}, we approximately minimize the cumulative logarithmic loss over the respective decision sets using entropic mirror descent with Armijo line search \citep{Li2019a}.

For \citet[Eq. (4.10)]{Howard2021}, we set $\eta=2$ and $s=1.4$, as in their Eq. (3.4), and $b=1$, since the generated random vectors satisfy $\|v_t\|_2=1$.
Their bound controls the estimation error $\|\widehat\Sigma_t-\Sigma\|_{\mathrm{op}}$, but its radius depends on $\|\Sigma\|_{\mathrm{op}}$, where $\Sigma$ is the unknown covariance matrix.
To obtain a computable confidence set, we evaluate the bound at each candidate covariance matrix $\mu$, retaining $\mu$ if
\[
\|\widehat\Sigma_t-\mu\|_{\mathrm{op}}<r_t(\mu),
\]
where $r_t(\mu)$ is the radius in their Eq. (4.10) with $\Sigma$ replaced by $\mu$.
This preserves the time-uniform coverage guarantee because, when $\mu=\Sigma$, the membership condition coincides with the original concentration inequality.

For \citet[Theorem 4.2]{Wang2025}, we construct the confidence sequence using a weighted average of the observed matrices.
We choose the weights following the authors' recommendation after their Corollary 4.3 for the case where the total number of observations is not known in advance, using their empirical matrix variance estimate in Eq. (27).
We also use their estimate $\widehat\xi_t$ of the next observation in Eq. (32), with $\bar V_0=I_d/4$ and $\widehat\xi_1=0$ as initial values.
To obtain a two-sided operator-norm confidence sequence, we apply their one-sided bound to both tails with error probability $\alpha/2$ each.

Let $u_j=e_j$ for $1 \leq j \leq d$, where $e_j$ denotes the $j$-th element in the standard basis in $\mathbb{R}^d$, and define
\[
u_{d+1}
=
\frac{(1,\ldots,d)^\transp}
     {\sqrt{\sum_{k=1}^d k^2}},
\qquad
u_{d+2}
=
\frac{\bigl((-1)^k(k+\tfrac12)\bigr)_{k=1}^d}
     {\sqrt{\sum_{k=1}^d(k+\tfrac12)^2}}.
\]
Each random vector is generated as $v_t=\epsilon_tu_{\iota_t}$, where $(\epsilon_t)_{t\in\mathbb{N}}$ is a sequence of i.i.d. Rademacher random variables and $(\iota_t)_{t\in\mathbb{N}}$ is an independent sequence of i.i.d. random variables with
\[
p_j \coloneqq \mathsf{P}(\iota_t=j) \propto \eu^{-0.38 ( j - 1 )} , \quad \text{for all } 1 \leq j \leq d + 2 . 
\]
Consequently, $\mathsf{E}[v_t]=0$, $\|v_t\|_2=1$, and
\[
\Sigma
=
\operatorname{diag}(p_1,\ldots,p_d)
+
p_{d+1}u_{d+1}u_{d+1}^\transp
+
p_{d+2}u_{d+2}u_{d+2}^\transp.
\]

We estimate the relative Hilbert--Schmidt volume by sampling candidate density matrices of the form $GG^\transp/\tr(GG^\transp)$, where $G\in\mathbb R^{d\times(d+1)}$ has independent standard Gaussian entries.
This construction generates matrices uniformly distributed with respect to the Hilbert--Schmidt volume measure on real density matrices \citep[Section 4, Eq. (25)]{Osipov2010}.

\end{document}